\documentclass[letterpaper, 10 pt, conference]{ieeeconf}
\usepackage{textcomp}

\usepackage{amsmath}
\usepackage{amssymb}
 \usepackage{amsthm}   
\usepackage{todonotes}
\usepackage[linesnumbered,ruled,vlined]{algorithm2e}
\SetKwInput{KwData}{Input}
\SetKwInput{KwResult}{Output}
\SetKwInput{KwParameters}{Parameters}
\usepackage[noadjust]{cite}
\usepackage{acronym}
\usepackage{subcaption}
\usepackage{soul}
\usepackage{multirow}
\usepackage{longtable,tabularx}
\usepackage{xcolor,colortbl}
\usepackage{color}

\newcommand{\agents}{\mathcal{A}}
\newcommand{\tasks}{\mathcal{T}}
\newcommand{\cspace}{\mathcal{X}}
\newcommand{\obst}{\cspace_{\textrm{obs}}}
\newcommand{\free}{\cspace_{\textrm{free}}}
\newcommand{\path}{\sigma}
\newcommand{\pathSet}{\Sigma}

\newcommand{\ijPath}{\pathSet_{i,j}}
\newcommand{\deltaPath}{\pathSet^\delta}
\newcommand{\minimise}{\textrm{minimise}}
\newcommand{\suchthat}{\textrm{subject to}}
\newcommand{\vertices}{\mathcal{V}}
\newcommand{\edges}{\mathcal{E}}

\newcommand{\bottAssign}{B}
\newcommand{\lambdal}{\underline{\lambda}}
\newcommand{\lambdau}{\overline{\lambda}}
\newcommand{\nmin}{n_\textrm{min}}
\newcommand{\nmax}{n_\textrm{max}}

\newcommand{\ant}{\agents \times \tasks}
\newcommand{\pathmin}{\underline{\sigma}}
\newcommand{\pathmax}{\overline{\sigma}}
\newcommand{\bigO}{\mathcal{O}}

\newcommand{\R}{\mathbb{R}}

\newcommand{\N}{\mathbb{N}}
\newcommand{\cl}{\operatorname{cl}}

\DeclareMathOperator*{\argmin}{argmin} 
 
\newtheorem{defn}{Definition}
\newtheorem{thm}{Theorem}
\newtheorem{assum}{Assumption}
\newtheorem{rem}{Remark}
\newtheorem{cor}{Corollary}
\newtheorem{prop}{Proposition}
\newtheorem{lem}{Lemma}
\newtheorem{req}{Requirement}
\newtheorem*{problem*}{Problem Statement}

\acrodef{bap}[BAP]{Bottleneck Assignment Problem}
\acrodef{prm}[PRM]{Probabilistic Road-Map}
\acrodef{rrt}[RRT]{Rapidly-exploring Random Tree}
\makeatletter
\let\NAT@parse\undefined
\makeatother
\usepackage[hidelinks]{hyperref} 

\def\BibTeX{{\rm B\kern-.05em{\sc i\kern-.025em b}\kern-.08em
    T\kern-.1667em\lower.7ex\hbox{E}\kern-.125emX}}
\begin{document}
\IEEEoverridecommandlockouts     
\overrideIEEEmargins                                      


\title{\LARGE \bf  Certification of Bottleneck Task Assignment with Shortest Path Criteria
}
\author{Tony A.\ Wood and Maryam Kamgarpour
\thanks{The authors are with the SYCAMORE Lab, École Polytechnique Fédérale de Lausanne (EPFL), 1015 Lausanne, Switzerland
        {\tt\small \{tony.wood,maryam.kamgarpour\}@epfl.ch}}%
}

\maketitle
\thispagestyle{empty}
\pagestyle{empty}

\begin{abstract}
Minimising the longest travel distance for a group of mobile robots with interchangeable goals requires knowledge of the shortest length paths between all robots and goal destinations. Determining the exact length of the shortest paths in an environment with obstacles is NP-hard however. In this paper, we investigate when polynomial-time approximations of the shortest path search are sufficient to determine the optimal assignment of robots to goals. In particular, we propose an algorithm in which the accuracy of the path planning is iteratively increased. The approach provides a certificate when the uncertainties on estimates of the shortest paths become small enough to guarantee the optimality of the goal assignment. To this end, we apply results from assignment sensitivity assuming upper and lower bounds on the length of the shortest paths. We then provide polynomial-time methods to find such bounds by applying sampling-based path planning. The upper bounds are given by feasible paths, the lower bounds are obtained by expanding the sample set and leveraging the knowledge of the sample dispersion. We demonstrate the application of the proposed method with a multi-robot path-planning case study.  
\end{abstract}

\section{Introduction}
Cooperative multi-robot systems provide great value in applications such as coordinated search and rescue, large-scale agriculture, and efficient transportation. Given a group of robots with interchangeable goals, deciding which one is assigned to which goal is crucial for achieving a joint objective. For instance, in a search and rescue mission robots should be assigned to goal destinations such that all possible locations of distressed humans can be visited in a minimal amount of time. Apart from the obvious incentive to complete a cooperative mission at minimum cost or time, an optimal goal assignment can also provide other benefits for multi-robot coordination such as inter-agent collision-avoidance guarantees, \cite{Turpin2014AR,MacAlpine2015CoAI,Wood2020RAL}. If the costs of sending robots to goals are known, optimally deciding which robot should go to which goal corresponds to a well-studied problem called task assignment, see \cite{Burkard2012Book} 
for an overview. Typically, the costs are dependent on the lengths of the shortest paths between each robot and each goal. When the robots have the same constant velocity, minimising the shortest path is equivalent to minimising the travel time. 

When the environment contains obstacles, finding the shortest obstacle-avoiding path function that links two locations is an infinite-dimensional and challenging optimisation problem. Finite-dimensional formulations involve constraints that are either smooth non-convex, see e.g.~\cite{Zhang2021TCST} or mixed-integer, see e.g.~\cite{Marcucci2021Arxiv}. While these methods are powerful in finding good solutions, the problem is NP-hard to solve optimality. There exist sampling-based approaches for shortest path search, see e.g. \cite{Karaman2011IJRR}, but they only converge to an optimal solution asymptotically as the number of samples and the computational complexity approach infinity. The problem of multi-agent goal assignment to minimise the shortest obstacle-avoiding paths is therefore hard and understudied.

We investigate when approximate knowledge of the shortest paths obtained in polynomial time is sufficient to determine an optimal assignment. We focus on the \ac{bap}, see e.g., \cite{Khoo2023TAC}, where the objective is to minimise the largest cost among the assigned agent-task pairs, referred to as the bottleneck. The \ac{bap} is particularly relevant for minimum-time requirements when a team of agents operates in parallel because the largest assigned agent-task cost then relates to the time of completing all tasks, \cite{Burkard2012Book}. In this work, we follow an iterative approach to find an assignment where the accuracy of the path planning is increased step by step. As a stopping criterion of the iterative increase of accuracy, we check if the assignment is guaranteed to be a minimiser of the \ac{bap} defined with respect to the true shortest paths. 

For multi-agent path planning with adjustable accuracy, we consider sampled roadmap graphs as they are a popular approach for obtaining obstacle-avoiding paths, see e.g. \cite{Kavraki1998ToRaA, LaValle2006Book, Karaman2011IJRR}. The key idea is to probe the configuration space with a desired number of samples and check if they can be connected without intersecting any obstacles. Intuitively, the complexity of the approach increases when more samples are considered. When the subset of the configuration space that has a safe distance from obstacles is known and the areas that are reachable from every configuration are easily computed, the special case of visibility-based roadmaps provides an efficient approach to optimal path planning \cite{Nissoux1999IROS}. In this paper, we consider a more general safe set setting where the visibility domain is not known. The most common roadmap algorithm for this setting, called \ac{prm}, utilises probabilistic sampling where samples are drawn from the configuration space at random. It has been shown to provide probabilistic completeness \cite{Kavraki1998ToRaA}, meaning that the probability of finding a feasible path, if one exists, goes to one when the number of samples goes to infinity. Variants of \ac{prm} \cite{Karaman2011IJRR} have been shown to also be asymptotically optimal, meaning that the obtained path converges to the shortest obstacle-avoiding path with probability one. In \cite{Janson2018IJRR} the convergence rate of such algorithms is investigated as a function of the sample dispersion. However, these results do not directly provide a converging lower bound on the shortest path length. 

To obtain a stopping criterion for the iterative algorithm we apply assignment sensitivity analysis. In particular, given estimates of the path lengths between all robots and goals and a bottleneck optimising assignment, we quantify how much error in the estimates can be tolerated in order for the assignment to remain optimal. Methods to quantify such sensitivity for different types of assignment objectives and perturbation models have been derived in \cite{Sotskov1995DAM, Nam2015ICRA, Michael2022EJOR}. While these methods have been used for post-assignment analysis, they have not been applied to determine whether the accuracy of the cost estimates is sufficient for knowing the true optimal assignment in a scenario where the estimates are being refined at the time the assignment is made. 

The contributions of this paper are twofold: 

    i) The first contribution is providing a generic approach for testing whether polynomial-time path planning is sufficient for guaranteeing optimal goal assignment. We apply recent results on bottleneck assignment sensitivity presented in \cite{Michael2022EJOR} and propose a novel iterative algorithm that increases the accuracy of the path planning step-wise and returns a certificate when an assignment can be obtained that is known to be optimal. We prove that if the certificate is returned, it is guaranteed that the assignment made with estimates of the path lengths is optimal for the true shortest paths. The generic approach relies on the existence of polynomial-time methods to determine upper and lower bounds on the shortest paths between all robots and goals that converge as the invested computational complexity increases.
    
    ii) The second contribution is deriving specific bounds on the shortest paths between robots and goals. 
    Given a map of obstacles and the required safety distance that robots must keep from them, we propose a sampling-based method to obtain both upper and lower bounds. The upper bounds are determined by finding feasible paths via a roadmap that is generated within the safe set. The lower bounds are obtained by generating a second roadmap that contains nodes representing positions that can be closer to the obstacles than the required safety distance. The distance from the obstacles is a function of the sample dispersion. Results from \cite{Janson2018IJRR} are then applied to these non-feasible paths in a novel manner to bound the optimal path lengths from below. We prove that the bounds converge to the true shortest safe paths as the sample size increases. 

\section{Shortest Bottleneck Path}\label{sec:problemAndApproach}
We begin by providing some definitions required to formulate the problem mathematically. Let $\cspace \subset \R ^ d$ be a compact Euclidean configuration space with dimension $d \in \N$. For a given margin, $\delta > 0$, and a subset of the configuration space, $\mathcal{C} \subset \cspace$, the $\delta$-interior of $\mathcal{C}$, denoted $\mathcal{C}^\delta := \{x \in \cspace \, | \, \inf_{y \in \cspace \setminus \mathcal{C}} \|x - y\|_2 \geq \delta\}$, is the set of all configurations that are at least a distance of $\delta$ away from $\cspace \setminus \mathcal{C}$. Given a closed set of obstacles, $\obst \subset \cspace$, we define the obstacle free-space as $\free := \cl(\cspace \setminus \obst)$, where $\cl(S)$ denotes the closure of set $S$. Let there be a set of agents, $\agents$, where each agent, $i \in \agents$, is a robot with initial configuration $p_i \in \free$. Cooperatively the agents are required to fulfil a set of tasks, $\tasks$, where each task, $j \in \tasks$, represents a goal configuration $g_j \in \free$. A path is a continuous function, $\path: [0,1] \rightarrow \cspace$, with bounded variation. Let $\pathSet$ denote the set of all paths and $c : \pathSet \rightarrow \R_{\geq 0}$ map a path to its arc length. For obstacle-avoiding path planning, we define the following two particular subsets of paths. 
\begin{defn}[Robot-goal path] \label{defn:agentTaskPath}
A path, $\path \in \Sigma$, connects robot $i \in \agents$ and goal $j \in \tasks$, with positions $p_i$ and $g_j$, respectively, if $\path(0) = p_i$ and $\path(1) = g_j$. The set of all paths connecting $i$ and $j$ is denoted by $\ijPath$.
\end{defn}
\begin{defn}[$\delta$-clearance path]\label{defn:deltaClearance}
For $\delta > 0$, a path, $\path \in \pathSet$, has $\delta$-clearance in $\mathcal{C} \subset \cspace$ if $\sigma(\tau) \in \mathcal{C}^{\delta}$ for all $\tau \in [0,1]$. The set of all paths with $\delta$-clearance in $\mathcal{C}$ is denoted by $\deltaPath_{\mathcal{C}}$.
\end{defn} 

We assume that there are no more goals than robots, i.e. $|\agents| \geq |\tasks|$, and that each robot is contained inside a ball centred at its reference position with radius $s > 0$. We then formulate the cooperative obstacle-avoiding path-planning problem for a scenario where the largest required travel distance of all assigned agents is to be minimised,
\begin{subequations}\label{eq:bottleneckPathPlanning}
\begin{align}
    \minimise \quad& \max_{(i,j)\in \ant} c(\path_{i,j}) \pi_{i,j} \\
    \suchthat \quad
    & \pi_{i,j} \in \{0,1\} & \forall (i,j) \in \ant, \label{eq:bottleneckPathPlanning_bin}\\
    &\sum_{i \in \agents} \pi_{i,j} = 1 & \forall j \in \tasks, \label{eq:bottleneckPathPlanning_tasks}\\
    &\sum_{j \in \tasks} \pi_{i,j} \leq 1 & \forall i \in \agents, \label{eq:bottleneckPathPlanning_agents}\\
    & \path_{i,j}\in \ijPath \cap \pathSet^s_{\free} & \forall (i,j) \in \ant.
\end{align}
\end{subequations}
The decision variables in \eqref{eq:bottleneckPathPlanning} are the assignment $\Pi := (\pi_{i,j})_{(i,j\in\ant)}$ and the path functions $(\path_{i,j})_{(i,j\in\ant)}$.

\subsection{Challenge of Finding Optimal Assignment}\label{sec:challenge}
If the length of the shortest robot-goal path with $s$-clearance in $\free$, denoted $c(\path_{i,j}^s)$, where 
\begin{align}\label{eq:optimalPaths}
    \path_{i,j}^{s} \in \argmin_{\path \in \ijPath \cap \pathSet^s_{\free}} c(\path),
\end{align} 
was given for all $(i,j) \in \ant$, then \eqref{eq:bottleneckPathPlanning} would reduce to an optimisation over the binary decision variables defining the assignment $\Pi$. This problem is the well-known \ac{bap}, see \cite{Burkard2012Book}, and is equivalent to a search for an optimal matching in a bipartite graph with weights $W = (c(\path_{i,j}^s))_{(i,j) \in \ant}$. Let $\bottAssign(W)$ be an operator that returns the set of optimal assignments that solve the \ac{bap}.
We note that for known weights the \ac{bap} can be solved efficiently, see \cite{Burkard2012Book}. 
However, the shortest paths, defined in \eqref{eq:optimalPaths}, are not known a-priori and are hard to find. There exist methods that estimate the shortest paths with finite complexity, e.g. optimisation with a convex approximation of the constraints. Some methods converge to the optimum as the complexity of the approach approaches infinity, e.g. \ac{prm}* \cite{Karaman2011IJRR}. While the true shortest paths are not found in polynomial time, the paths returned by the algorithms with finite complexity may provide sufficiently good estimates to obtain an optimal assignment. 

We therefore would like to know how much uncertainty in the path lengths can tolerated for each robot-goal pair in order for a considered assignment to be optimal. This corresponds to quantifying allowable perturbations to nominal weights that are given by path-length estimates.
\begin{defn}[Allowable perturbation]\label{defn:allowablePerturbations}
Given a family of weights, $W = (w_{i,j})_{(i,j)\in\agents\times\tasks}$, with $w_{i,j} \in \R$, let $\Pi \in \bottAssign(W)$ be a bottleneck assignment. A family of perturbations $V=(v_{i,j})_{(i,j)\in\agents\times\tasks}$, with $v_{i,j} \in \R$, is allowable with respect to $\Pi$ for $W$ if $\Pi \in \bottAssign(W + V)$, where $W + V = (w_{i,j} + v_{i,j})_{(i,j)\in\agents\times\tasks}$. Let $\Lambda=([-\lambdal_{i,j},\lambdau_{i,j}])_{(i,j)\in\agents\times\tasks}$ be a family of intervals. A family of perturbations, $V = (v_{i,j})_{(i,j) \in \ant}$, is contained in $\Lambda$, denoted $V \in \Lambda$, if for all $(i,j) \in \ant$ we have $v_{i,j} \in [-\lambdal_{i,j}, \lambdau_{i,j}]$. The family of intervals $\Lambda$ is allowable relative to assignment $\Pi$ for weights $W$ if for all perturbations $V \in \Lambda$ it is guaranteed that $V$ is an allowable perturbation with respect to $\Pi$ for $W$.
\end{defn}

\subsection{Certification Algorithm}\label{sec:certification}
Algorithm~\ref{alg:approach} outlines the approach we follow to obtain an assignment. Given a considered accuracy, parameterised by $n$, the algorithm involves computing families of upper and lower bounds on shortest path lengths for every robot-goal pair with subroutines ${\tt Lower}$ and ${\tt Upper}$, respectively. The averages of the upper and lower bounds are used as assignment weights to determine a candidate assignment with subroutine ${\tt BottleneckAssignment}$. To determine whether the returned assignment is guaranteed to be optimal for the true shortest paths a subroutine ${\tt AllowableIntervals}$ is executed. If for all agent-task pairs the upper and lower bounds lie within the range of perturbed weights defined by the allowable intervals, the algorithm terminates with a certificate, $Q={\tt true}$. Otherwise, the accuracy is increased and the procedure is repeated until either a certificate is found or the maximal complexity defined by $\nmax$ is reached.

\SetKwComment{Comment}{/* }{ */}
\begin{algorithm}

\DontPrintSemicolon
\caption{Iterative planning and assignment}\label{alg:approach}
\KwData{Agent~positions~$P = (p_i)_{i \in \agents}$, goal~positions~$G = (g_j)_{j \in \tasks}$, free~space~$\free$, safety~distance~$s$;}
\KwResult{Assignment~$\Pi$, certificate~$Q$;}
\KwParameters{Initial~planning~parameter~$\nmin \in \N$, planning~parameter~limit~$\nmax\in\N$, increase~factor~$\alpha > 0$;}
 $Q \gets {\tt false}$\;
 $n \gets \nmin$\;
 \While{$Q = {\tt false}$ and $n \leq \nmax$}
 {
     $
     (u_{i,j})_{(i,j) \in \ant}  
     \gets {\tt Upper}(P,G,\free,s,n)$\;
     $ 
     (l_{i,j})_{(i,j) \in \ant} 
     \gets {\tt Lower}(P,G,\free,s,n)$\;
     $W \gets  \left(\frac{l_{i,j} + u_{i,j}}{2}\right)_{(i,j) \in \ant}$ \;  
     $\Pi \gets {\tt BottleneckAssignment}\left(W\right)$\;
     $ 
     \Lambda \gets {\tt AllowableIntervals}(W,\Pi)$\;
     \eIf{$\left(\left[\frac{l_{i,j}-u_{i,j}}{2},\frac{u_{i,j}-l_{i,j}}{2}\right]\right)_{(i,j) \in \ant} \subseteq \Lambda$}
     {
     $Q \gets {\tt true}$\;
     }
     {
     $n \gets \alpha \cdot n$\;
     }
 }
\end{algorithm}

Several methods of obtaining upper and lower bounds for the optimal costs, $(c(\sigma_{i,j}^s))_{(i,j) \in \ant}$, are conceivable. Such bounds should converge as the planning parameter $n$ is increased. In Section~\ref{sec:samplingBasedPlanning} we focus on a sampling-based method where the planning parameter is given by the sampling size $n$. The requirements for a generalisation to other methods are summarised in the following. 
\begin{req}\label{req:UpperConverges}\label{req:LowerConverges}\label{req:BoundsConverges}
The complexities of ${\tt Upper}$ and ${\tt Lower}$ are polynomial in the number of robots and planning parameter $n$. Furthermore, $\lim_{n \rightarrow \infty} u_{i,j} = \lim_{n \rightarrow \infty} l_{i,j} = c(\sigma_{i,j}^s)$ while $u_{i,j} \geq c(\sigma_{i,j}^s)$ and $l_{i,j} \leq c(\sigma_{i,j}^s)$ for all $\ant$.
\end{req}

 We use an edge removal algorithm introduced in~\cite{Khoo2023TAC} to implement ${\tt BottleneckAssignment}$ but any other approach to solve a \ac{bap}, see \cite{Burkard2012Book}, can be used. To implement ${\tt AllowableIntervals}$ we use the method derived in \cite{Michael2022EJOR}. This method performs a sensitivity analysis to determine how much every weight can independently be altered while guaranteeing that a considered optimising assignment remains optimal. As an output, it produces the lexicographic maximal family of intervals that are allowable relative to a given assignment and weights according to Definition~\ref{defn:allowablePerturbations}. 
 
\subsection{Theoretical Guarantees}\label{sec:guarantees}
\begin{assum}\label{assum:existsPath}
For safety distance, $s>0$, obstacle map, $\obst\subset \cspace$, all robots $i \in \agents$, and all goals $j \in \tasks$ there exists a connecting path with $s$-clearance, as defined in Definitions~\ref{defn:agentTaskPath} and~\ref{defn:deltaClearance}, i.e. the set $\ijPath \cap \pathSet^s_{\free}$ is non-empty. 
\end{assum} 
\begin{thm}
\label{thm:main}
Given Assumption~\ref{assum:existsPath}, assume Requirements~\ref{req:BoundsConverges} is satisfied.
If Algorithm~\ref{alg:approach} terminates with $Q = {\tt true}$, then the returned assignment, $\Pi$, is bottleneck minimising for the shortest paths between robots and goals, i.e., $\Pi^* = \Pi$ is an optimiser of \eqref{eq:bottleneckPathPlanning}.  
\end{thm}
\begin{proof}
Consider an arbitrary iteration of the \textbf{while}-loop given in Lines 3-12 of Algorithm~\ref{alg:approach}. Because Requirement~\ref{req:BoundsConverges} is satisfied, we have $(c(\path^s_{i,j}))_{(i,j) \in \ant} \in ([l_{i,j},u_{i,j}])_{(i,j) \in \ant}$, where $(u_{i,j})_{(i,j) \in \ant}$ and $(l_{i,j})_{(i,j) \in \ant}$ are determined in Lines~4 and~5. In Line~7, $\Pi$ is determined such that it is a bottleneck assignment for weights $W = (w_{i,j})_{(i,j) \in \ant}$, with $w_{i,j} = \frac{1}{2}(l_{i,j} + u_{i,j})$ according to Line~6. From \cite{Michael2022EJOR} and Definition~\ref{defn:allowablePerturbations}, we know that $\Pi \in \bottAssign(W + V)$ for all $V \subseteq \Lambda$, where $\Lambda$ is an allowable family of intervals determined in Line~8. If the algorithm returns $Q={\tt true}$, it means that the condition on Line~9 is satisfied and the difference between the weights and the true shortest paths corresponds to an allowable perturbation. Then, we have that $\Pi \in \bottAssign((c(\path^s_{i,j}))_{(i,j)\in\ant})$.
\end{proof}

From Theorem~\ref{thm:main} we know that Algorithm~\ref{alg:approach} returning a certificate is a sufficient condition for assignment optimality. Note that given a finite complexity limit, parameterised by $\nmax$, the algorithm may terminate with $Q = {\tt false}$ in which case no statement about the assignment optimality can be made. The individual subroutines of the algorithm have been shown to have polynomial complexity. Next, we show that if a significantly large increase in the value of the planning parameter from one iteration to the next is chosen, then the complete algorithm terminates in polynomial time.
\begin{prop}
\label{prop:polynomialComplexity}
Given Assumption~\ref{assum:existsPath}, assume Requirements~\ref{req:BoundsConverges} is satisfied. If the increase factor is such that 
\begin{align}\label{eq:minIncrease}
    \alpha > \left(\frac{\nmax}{\nmin}\right) ^ {\frac{1}{\nmax}},
\end{align} then Algorithm~1 has a computational complexity that is polynomial in parameter $\nmax$ and the number of robots $|\agents|$.
\end{prop}
\begin{proof}
    For a given parameter $n$ and at most $|\agents|$ tasks, let $C_U(|\agents|,n)$ denote the complexity of ${\tt Upper}$, $C_L(|\agents|,n)$ denote the complexity of ${\tt Lower}$, $C_B(|\agents|)$ denote the complexity of the bottleneck assignment, and $C_A(|\agents|)$ denote the complexity of computing the allowable intervals. If \eqref{eq:minIncrease} is satisfied, there will be at most $\nmax$ iterations of the \textbf{while}-loop on Lines~3-12 of Algorithm~\ref{alg:approach}. Therefore, the worst-case complexity is $\bigO(\nmax(C_U(|\agents|,\nmax) + C_L(|\agents|,\nmax) + C_B(|\agents|) + C_A(|\agents|)))$. We know that $C_B(|\agents|)$ and $C_A(|\agents|)$ are polynomial in $|\agents|$ from \cite{Pentico2007EJoOR} and \cite{Michael2022EJOR}, respectively. Given the satisfaction of Requirement~\ref{req:BoundsConverges}, it follows that the algorithm is polynomial in $|\agents|$ and $\nmax$. 
\end{proof}

\section{Sampling-Based Path Planning}\label{sec:samplingBasedPlanning}

\subsection{Background on Sampled Roadmaps}\label{sec:roadmapBackground}
Sampling-based planning consists of finding paths between desired starting and goal points by connecting samples of the configuration space. A popular approach that is suited for multi-query path planning, e.g. the planning for multiple robots and goals simultaneously, consists of building a so-called roadmap. A roadmap is a graph with vertices representing points in a desired set, $\mathcal{C}\subseteq\cspace$, and edges representing lines contained in $\mathcal{C}$ that connect such points.  The roadmap-building approach outlined in Algorithm~\ref{alg:roadmap} (closely related to sPRM in \cite{Karaman2011IJRR} and gPRM in \cite{Janson2018IJRR}) relies on the following subroutines: ${\tt Sample}$ first generates a set of $n \in \N$ samples of the configuration space, $\mathcal{N}:=\{x_1,\dots x_n\} \subset \cspace$ and then returns the subset of these samples that are in $\mathcal{C}$, i.e, $S = \mathcal{N} \cap \mathcal{C}$; ${\tt Near}$ returns the set of nodes within a radius $r$ of node $v$, i.e, $\cspace_{\textrm{Near}} = \{x \in \vertices \setminus \{v\} \,|\, \|x - v\|_2 < r\}$; ${\tt UninterruptedEdge}$ returns ${\tt true}$ if and only if the linear interpolation between nodes $u$ and $v$ lies entirely in the sample space $\mathcal{C}$. For details, we refer to the discussion of the function ${\tt CollionFree}$ in \cite{Karaman2011IJRR}. 

\begin{algorithm}
\DontPrintSemicolon
\caption{${\tt RoadMap}$}\label{alg:roadmap}
\KwData{Agent~positions~$P = (p_i)_{i \in \agents} $, goal~positions~$G = (g_j)_{j \in \tasks}$, sample~space~$\mathcal{C}$, connection~radius~$r$, sample~size~$n$;}
\KwResult{Rode-map~nodes~$\vertices$, roadmap~edges~$\edges$;}
$\mathcal{S} \gets {\tt Sample}(\mathcal{C},n)$\;
$\vertices \gets \mathcal{S} \cup P \cup G $\;
$\edges \gets \emptyset$\;
\For {$v \in \vertices$}
{
$\cspace_{\textrm{Near}} \gets {\tt Near}(\vertices,v,r)$\;
\For {$x \in \cspace_{\textrm{Near}}$}
{
\If {${\tt UninterruptedEdge}(\mathcal{C},v,x)$}
{
$\edges \gets \edges \cup \{(v,x)\} \cup \{(x,v)\}$
}
}
}
\end{algorithm}

If the graph returned by Algorithm~\ref{alg:roadmap} connects nodes $p_i$ and $g_j$, then a routine for finding the shortest path in a graph, such as Dijkstra's 
algorithm, can be applied. The resulting path remains in $\mathcal{C}$ and connects robot $i \in \agents$ and goal $j \in \tasks$, i.e. $\hat{\path}_{i,j} \gets {\tt ShortestPath}(p_i,g_j,\vertices,\edges)$, where $\hat{\path}_{i,j} \in \ijPath \cap \pathSet^0_{\mathcal{C}}$. Crucially, this procedure can be run in polynomial time. \begin{rem}[Shown in \cite{Karaman2011IJRR}]\label{rem:complexityRoadmap}
The computational complexities of ${\tt RoadMap}$ and one query of Dijkstra's algorithm in the resulting graph are each $\bigO(n^2)$.
\end{rem}

In \cite{Janson2018IJRR} conditions for feasibility, see Remark~\ref{rem:minRadius}, and accuracy, see Theorem~\ref{thm:optimalityBound}, of the roadmap-based path planning were derived as functions of the sample dispersion.  
\begin{defn}[$l_2$-dispersion]\label{defn:dispersion}
For a finite non-empty set $\mathcal{S} \subset \cspace$, the dispersion of $\mathcal{S}$ in the compact set $\mathcal{C} \subset \cspace$, with positive Lebesgue measure, is 
\begin{align*}
    D(\mathcal{C},\mathcal{S}) := \sup_{c \in \mathcal{C}} \min_{s \in \mathcal{S}} \|s - c\|_2.
\end{align*}
\end{defn}
The dispersion of the node set generated in Line~2 of Algorithm~\ref{alg:roadmap}, i.e. $D(\mathcal{C},\vertices)$, can be described as the radius of the largest ball in $\mathcal{C}$ that does not contain a node $v \in \vertices$.
\begin{rem}[Shown in \cite{Janson2018IJRR}]\label{rem:minRadius}
If the connection radius in Algorithm~\ref{alg:roadmap} is selected such that $r > 2 D(\mathcal{C},\vertices)$, then the graph returned by Algorithm~\ref{alg:roadmap} not connecting $p_i$ and $g_j$ means that there does not exist a path for robot $i \in \agents$ and goal $j \in \tasks$ with $\delta$-clearance, for any $\delta \geq 2 D(\mathcal{C},\vertices)$.    
\end{rem}
The implication of this statement is that the dispersion should be small in relation to the connection radius $r$ and a parameter $\delta$ that quantifies a desired minimum distance from the boundary of the sample set $\mathcal{C}$. 

\begin{assum}\label{assum:rightRadius}
The connection radius, $r$, in Algorithm~\ref{alg:roadmap}, is selected such that $r \in \left(2 D(\mathcal{C},\vertices), \delta - D(\mathcal{C},\vertices)\right)$.
\end{assum}
We note that for Assumption~\ref{assum:rightRadius} to be satisfied, we must have $\delta > 3 D(\mathcal{C},V)$.
\begin{thm}[proven in \cite{Janson2018IJRR}]\label{thm:optimalityBound}
For a given margin, $\delta > 0$, assume there exists a path, $\path \in \ijPath \cap \pathSet^\delta_{\mathcal{C}}$, connecting robot $i \in \agents$ and goal $j\in \tasks$ with $\delta$-clearance as defined in Definitions~\ref{defn:agentTaskPath} and~\ref{defn:deltaClearance}. Let $c(\hat{\path}_{i,j})$ be the length of the path between $p_i$ and $g_j$ returned by ${\tt shortestPath}(p_i,g_j,\vertices,\edges)$, where $(\vertices,\edges)$ is the graph obtained by Algorithm~\ref{alg:roadmap} on samples with $l_2$-dispersion of $D(\mathcal{C},\vertices)$ and a connection radius, $r$, that satisfies Assumption~\ref{assum:rightRadius}. Then, we have
\begin{align*}
    c(\hat{\path}_{i,j}) \leq \left(1 + \frac{2D(\mathcal{C},\vertices)}{r - 2D(\mathcal{C},\vertices)}\right) \min_{\path \in \ijPath \cap \pathSet^\delta_{\mathcal{C}}} c(\path). 
\end{align*} 
\end{thm}

The sample dispersion depends on the sampling scheme applied in the subroutine ${\tt Sample}$. Intuitively, the more samples considered the smaller the dispersion is, i.e, $D(\mathcal{C},\vertices)$ decreases for increasing $n$. Given just the set of samples, it is difficult to determine the dispersion. However, if the samples are generated with a deterministic procedure, a bound on the dispersion, $\hat{D} \geq D(\mathcal{C},\vertices)$, can be computed. In the case where $\mathcal{C} = [0,1]^d$ and the samples are obtained by gridding with a cubic lattice with $n=k^d$ uniformly spaced grid points and $k \in \N$, the dispersion is bound by $\hat{D} = 0.5 d^{\frac{1}{2}} n^{-\frac{1}{d}}$, see \cite{Sukharev1971CMMP}. For the two-dimensional case, i.e., $\mathcal{C} = [0,1]^2$ the lowest possible dispersion bound of $\hat{D} = 0.62 n^{-\frac{1}{d}}$ is achieved with triangular tiling, see \cite{LaValle2006Book}. In probabilistic sampling-based motion planning algorithms, such as \ac{prm} and \ac{prm}$^*$ \cite{Karaman2011IJRR}, the samples are randomly drawn from a probability distribution. For random samples drawn from a known distribution, probabilistic properties of the dispersion can be expressed. The dispersion of $n$ independently uniformly sampled points on $[0,1] ^ d$ is $\bigO (\log(n)^{\frac{1}{d}}n^{-\frac{1}{d}})$ with probability 1, see \cite{Deheuvels1983ZWVG}. 

\subsection{Implementation of ${\tt Upper}$
}\label{sec:Upper}
To obtain upper bounds for the shortest paths with $s$-clearance, we apply the subroutine described in Algorithm~\ref{alg:roadmap} and generate a feasible path, $\pathmax_{i,j} \in \ijPath \cap \pathSet^s_{\free}$ for all $(i,j) \in \ant$ via a roadmap graph $(\overline{\vertices},\overline{\edges})$ created by sampling from $\mathcal{C} = \free^s$. Based on the properties discussed in Section~\ref{sec:roadmapBackground}, the connection radius, $r$, is selected such that it decreases with the dispersion bound, $\hat{D}$, for increasing sample sizes, $n$, such that Assumption~\ref{assum:rightRadius} is satisfied. If the dispersion is sufficiently small, paths are found for every robot-goal pair. The solid blue lines in Figure~\ref{fig:bound_example} illustrate examples of such paths for varying sample sizes. The sample set, $\free^s$, is represented by the white area that consists of all points that have a greater distance from all obstacles than $s$.

\begin{figure*}
    \begin{subfigure}{0.34580645161\linewidth}
        \centering
        \includegraphics[width=\linewidth]{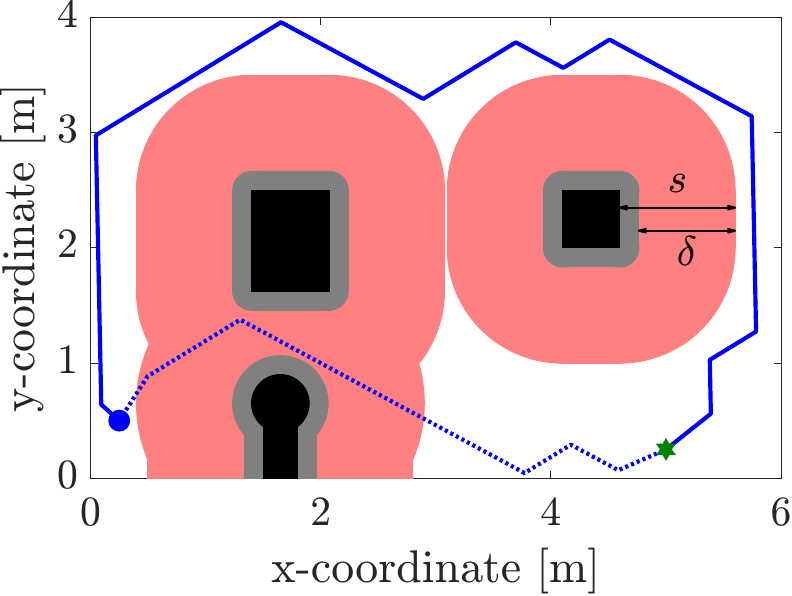}
        \caption{Sample size $n = 128$} 
        \label{fig:bound_example_1}
    \end{subfigure}
    \hfill
    \begin{subfigure}{0.32\linewidth}
        \centering
        \includegraphics[width=\linewidth]{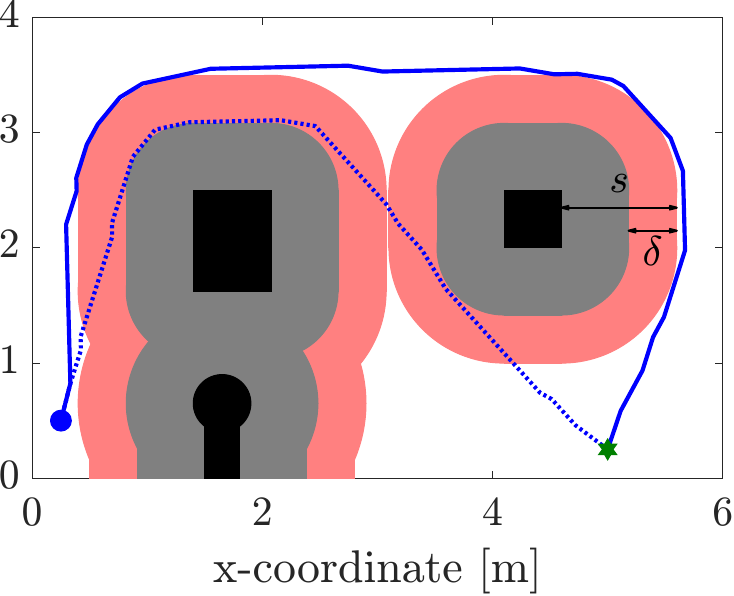}
        \caption{Sample size $n = 2048$} 
        \label{fig:bound_example_2}
    \end{subfigure}
    \hfill
    \begin{subfigure}{0.32\linewidth}
        \centering
        \includegraphics[width=\linewidth]{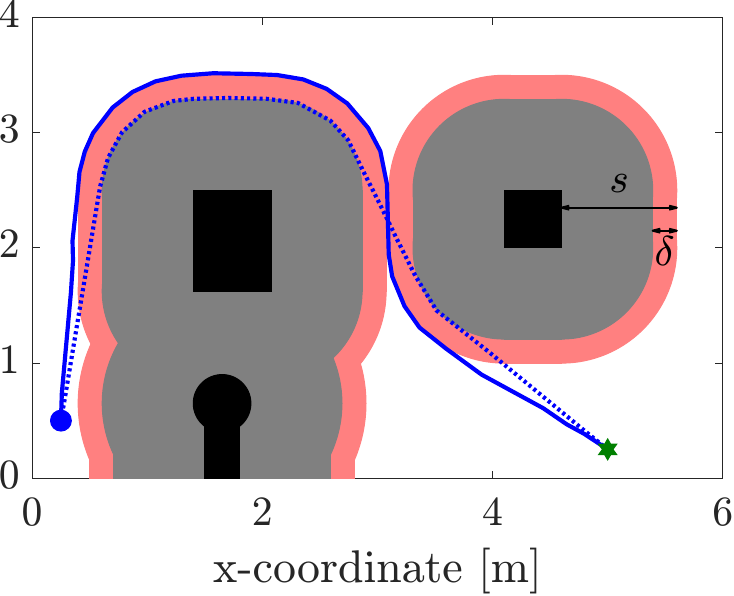}
        \caption{Sample size $n = 32768$} 
        \label{fig:bound_example_3}
    \end{subfigure}
    \caption{Shortest paths generated from roadmaps $(\overline{\vertices},\overline{\edges})$ [solid blue line] and $(\underline{\vertices},\underline{\edges})$ [dotted blue line] with robot position [blue dot], goal [green star], obstacles [black area], expansion of obstacles by $s$ [red area], and expansion of obstacles by $s - \delta$ [grey area], for parameter values $\zeta = 0.5$, $\eta = 2.2 \cdot 10^{-16}$, and samples obtained from triangular tiling.}  
    \label{fig:bound_example}
\end{figure*}

The arc lengths of the feasible paths are then used as upper bounds on the optimal path costs, i.e., $(u_{i,j})_{(i,j)\in \ant} = (c(\pathmax_{i,j}))_{(i,j)\in \ant} \geq (c(\sigma^s_{i,j}))_{(i,j)\in \ant}$. In \cite{Janson2018IJRR} it is shown, based on Theorem~\ref{thm:optimalityBound}, that these paths converge to the optimal paths asymptotically 
as $n \rightarrow \infty$. We know from Remark~\ref{rem:complexityRoadmap} that the complexity of finding these upper bounds is $\bigO(|\agents|^2 n^2)$. To fulfil the second half of Requirement~\ref{req:LowerConverges} we derive lower bounds next.   

\subsection{Implementation of ${\tt Lower}$}\label{sec:Lower}
To determine lower bounds on the shortest paths with $s$-clearance, we consider an alternative roadmap graph, $(\underline{\vertices},\underline{\edges})$. It is generated by expanding the space from which samples are taken to include some points that are outside of the $s$-interior of $\free$, i.e., in the $(s - \delta)$-interior of $\free$ with $\delta \in [0,s]$. In other words, we allow for samples that are closer to obstacles than $s$. We note that the path constructed from such samples may not be feasible. However, by regulating the amount of the extension of the sample space with parameter $\delta$, lower bounds on the shortest feasible paths are obtained. 
If the dispersion of the nodes in the roadmap has a known upper bound, $\hat{D} \geq D(\free^{s-\delta},\underline{\vertices})$, we can bound the optimal path lengths with the following corollary derived from Theorem~\ref{thm:optimalityBound}.

\begin{cor}\label{cor:lowerBound}
    If Assumptions~\ref{assum:existsPath} holds, the connection radius, $r$, is selected such that Assumption~\ref{assum:rightRadius} is satisfied for a margin, $\delta \in [0,s]$, then we have  
\begin{align}\label{eq:accuracyParameter}
    \left(1 - \frac{2 \hat{D}}{r}\right) c(\pathmin_{i,j}) \leq  c(\path_{i,j}^s) ,  
\end{align}
where $\pathmin_{i,j}$ is the shortest path in roadmap $(\underline{\vertices},\underline{\edges})$ connecting $p_i$ and $g_j$, the optimal path $\path^s_{i,j}$ is defined in \eqref{eq:optimalPaths}, and $\hat{D} \geq D(\free^{s-\delta},\underline{\vertices})$ is any upper bound on the sample dispersion. 
\end{cor}

Algorithm~\ref{alg:Lower} is a specific procedure to implement ${\tt Lower}$. Given knowledge of the sampling scheme, see Section~\ref{sec:roadmapBackground}, a bound is computed in subroutine ${\tt DispersionBound}$ as a function of the sample size. 
If the dispersion bound is too large, i.e., $3 \hat{D} \geq s$, then Algorithm~\ref{alg:Lower} returns infinitely low lower bounds on the optimal path lengths. If the dispersion bound is sufficiently small, i.e., $3 \hat{D} < s$, then the shortest paths in the roadmap graph $(\underline{\vertices},\underline{\edges})$ are used to compute lower bounds on the optimal path lengths based on \eqref{eq:accuracyParameter}. Analogue to the procedure for ${\tt Upper}$ the connection radius is selected such that it decreases with the dispersion bounds as sample size $n$ increases. However, in the case of ${\tt Lower}$ the set from which the samples are drawn, $\free^{s-\delta}$, varies as it shrinks towards $\free^{s}$ with increasing $n$. The dotted blue lines in Figure~\ref{fig:bound_example} illustrate the path obtained from $(\underline{\vertices},\underline{\edges})$ for varying sample sizes. The sample set, $\free^{s-\delta}$, is the union of the white and red areas consisting of all points that have a greater distance from all obstacles than $s-\delta$.    

\begin{algorithm}
\DontPrintSemicolon
\caption{${\tt Lower}$}\label{alg:Lower}
\KwData{Initial~positions~$P = (p_i)_{i \in \agents} $, goal~positions~$G = (g_j)_{j \in \tasks}$, free~space~$\free$, safety~distance~$s$, sample~size~$n$;}
\KwResult{Lower~bounds~$L = (l_{i,j})_{(i,j) \in \ant}$;}
\KwParameters{Margin~tuning~parameter~$\zeta \in (0,1)$, radius~tuning~parameter~$\eta \in (0,1)$;}
$\hat{D} \gets {\tt DispersionBound}(\free,n)$\;
\eIf{$\hat{D} \geq \frac{s}{3}$ }
{
$(l_{i,j})_{(i,j) \in \ant} \gets (-\infty)_{(i,j) \in \ant}$\;
}
{
$\delta \gets (3\hat{D}) ^ {\zeta}  s ^ {1 - \zeta}$\;
$r \gets \eta 2 \hat{D} + (1 - \eta) (\delta -  \hat{D})$\;
$(\underline{\vertices},\underline{\edges}) \gets {\tt RoadMap}(P,G,\free^{s - \delta},r,n)$\;
$\beta \gets 1 - \frac{2\hat{D}}{r}$\;
\For{$(i,j) \in \ant$}
{
    $\pathmin_{i,j} \gets {\tt ShortestPath}(p_i,g_j,\underline{\vertices},\underline{\edges})$\;
    $l_{i,j} \gets \beta c(\pathmin_{i,j})$\;
}
}
\end{algorithm}

From Remark~\ref{rem:complexityRoadmap} we know that the complexity of Algorithm~\ref{alg:Lower} is also $\bigO(|\agents|^2 n^2)$. We conclude this section by showing in Lemma~\ref{lem:lowerBounds} that Requirement~\ref{req:LowerConverges} can be satisfied with this sampling-based implementation.

\begin{lem}
\label{lem:lowerBounds}
The values $(l_{i,j})_{(i,j)\in\ant}$ returned by Algorithm~\ref{alg:Lower} are lower bounds on the optimal paths $(c(\path_{i,j}^s))_{(i,j)\in\ant}$, defined in \eqref{eq:optimalPaths}, and converge to the optimal values if $\hat{D}\geq D(\free^{s-\delta},\underline{\vertices})$ and $\hat{D} \rightarrow 0$ as $n \rightarrow 0$.
\end{lem}
\begin{proof}
If $s \leq 3 \hat{D}$, Algorithm~\ref{alg:Lower} returns $(l_{i,j})_{(i,j) \in \ant} = (-\infty)_{(i,j) \in \ant}$ and we have $l_{i,j} < c(\sigma^s_{i,j})$ for all $(i,j) \in \ant$. 
If $s > 3 \hat{D}$, then from Lines~5 and~6, we have $\delta > 3 \hat{D} \geq 3 D(\free^{s-\delta},\underline{\vertices})$ and $r > 2 \hat{D}(n) \geq 2 D(\free^{s-\delta},\underline{\vertices})$. We also have $r < \delta - \hat{D} \leq \delta - D(\free^{s-\delta},\underline{\vertices})$ and Assumption~\ref{assum:rightRadius} is therefore satisfied. From Assumption~\ref{assum:existsPath} and the fact that $\free^s \subset \free^{s-\delta}$ we know that there exists a path in $\pathSet^{s-\delta}_{\free}$ that connects $p_i$ and $g_j$ for all $(i,j) \in \ant$. From Remark~\ref{rem:minRadius} it follows that there exists a path in $(\underline{\vertices},\underline{\edges})$ that connects $p_i$ and $g_j$ for all $(i,j) \in \ant$. The path $\pathmin_{i,j}$ determined in Line~10 is the shortest path in $(\underline{\vertices},\underline{\edges})$ connecting $p_i$ and $g_j$. From Corollary~\ref{cor:lowerBound} it follows that $l_{i,j} =  (1-\frac{2\hat{D}}{r})c(\pathmin_{i,j}) \leq  c(\path^s_{i,j})$ for all $(i,j) \in \ant$.
Moreover, we have $\frac{\hat{D}}{r} \rightarrow 0 $ and $\delta \rightarrow 0$ as $n \rightarrow \infty$. Thus, $\free^{s-\delta} \rightarrow \free^s$ and for all $(i,j) \in \ant$ $c(\pathmin_{i,j}) \rightarrow c(\path_{i,j}^s)$ as $n \rightarrow \infty$.  
\end{proof}

\section{Numerical Analysis}\label{sec:caseStudy}

\subsection{Case Study}
\label{sec:results}
We consider agents, $\agents$, representing 5 ground robots. The positions of the robots $p_1, \dots p_5 \in \cspace$ are illustrated in Figure~\ref{fig:case_study_0} in blue. The configuration space contains obstacles, $\obst\subset\cspace$ shown in black, from which the robot positions must keep a safety distance of at least $s$, indicated in red.
The cooperative mission is to visit all 3 goals, $g_1,g_2,g_3 \in \cspace$ shown in green, that represent the task set, $\tasks$. 

\begin{figure*}
    \centering
    \begin{subfigure}{0.25875\linewidth}
        \centering
        \includegraphics[width=\linewidth]{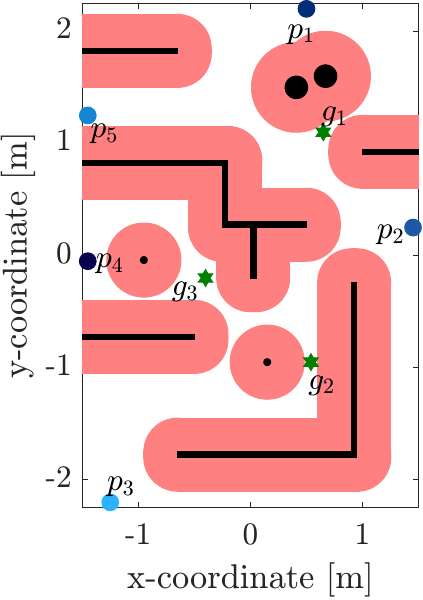}
        \caption{Initial configuration} 
        \label{fig:case_study_0}
    \end{subfigure}
    \hfill
    \begin{subfigure}{0.23\linewidth}
        \centering
        \includegraphics[width=\linewidth]{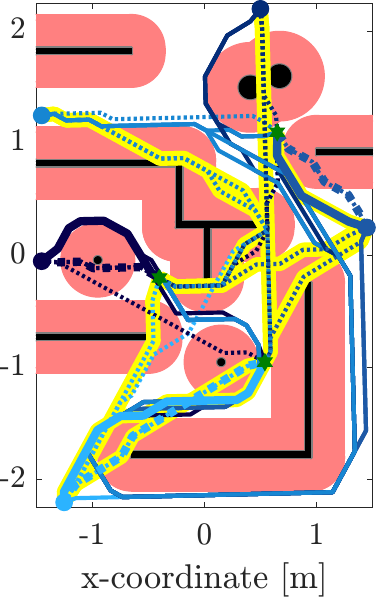}
        \caption{Iteration 1 ($n = 1024$)}
        \label{fig:case_study_1}
    \end{subfigure}
    \hfill
    \begin{subfigure}{0.23\linewidth}
        \centering
        \includegraphics[width=\linewidth]{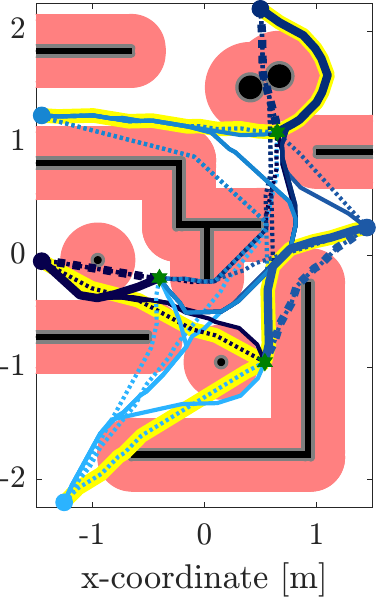}
        \caption{Iteration 2 ($n = 4168$)}
        \label{fig:case_study_2}
    \end{subfigure}
    \hfill
    \begin{subfigure}{0.23\linewidth}
        \centering
        \includegraphics[width=\linewidth]{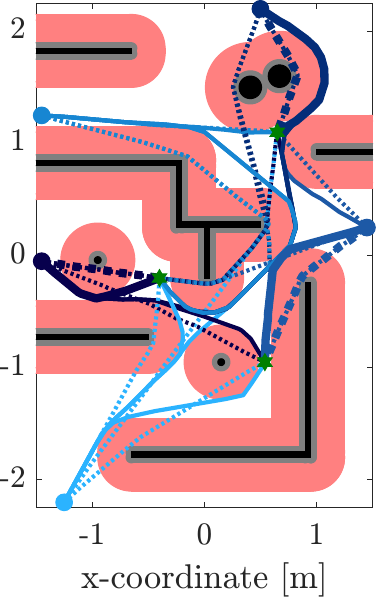}
        \caption{Iteration 3 ($n = 16672$)}
        \label{fig:case_study_3}
    \end{subfigure}
    \caption{Iterative multi-query path planning with robot positions [blue dots], goals [green stars], obstacles [black area], expansion of obstacles by $s$ [red area], and expansion of obstacles by $s-\delta$ [grey area]. Shortest paths generated from roadmaps $(\overline{\vertices},\overline{\edges})$ [solid blue lines] and $(\underline{\vertices},\underline{\edges})$ [dotted blue lines] are shown, where assigned robot-goal pairs [thick] and paths corresponding to assignment weights that are not contained in the allowable intervals [yellow background] are highlighted.}
    \label{fig:case_study}
\end{figure*}

Figure~\ref{fig:case_study} shows the results of Algorithm~\ref{alg:approach} using a deterministic sampling based on triangular tiling \cite{Janson2018IJRR}. The initial sample size for the path length bounding schemes is selected to be $\nmin = 1024$ and the sample increase factor is $\alpha = 4$. The maximum sample size is set to $\nmax = 66688$. We note for this choice of parameters we have $3\hat{D} < s$ in the first iteration and \eqref{eq:minIncrease} is satisfied. The margin tuning parameter is set to $\zeta = 0.1$ and the radius tuning parameter is set to $\eta = 0.1$. In each iteration of the \textbf{while}-loop, two paths are generated for each robot-goal pair: one, shown with a solid line, lies in $\free^s$ based on roadmap graph $(\overline{\vertices},\overline{\edges})$ and one, shown with a dotted line, lies in $\free^{s - \delta}$ based on roadmap graph $(\underline{\vertices},\underline{\edges})$. The lengths of these paths are used to determine the upper and lower bounds on the shortest paths, where the lower bound computation includes the scaling factor $\beta= 1-\frac{2\hat{D}}{r}$. The candidate assignment is made based on the averages of the upper and lower bounds. The resulting robot-goal pairings are illustrated by thick lines. The allowable ranges of path lengths for the optimality of the candidate assignments are given in Table~\ref{tab:allowableIntervals}. The numerical values of the other evolving variables are listed in Table~\ref{tab:values_diffvalues}. Note that the computation time, $t$, is provided for a non-optimised code implementation in Matlab where the bounds for all 15 agent-tasks pairs are computed in series.          

\begin{table*}
    \caption{\label{tab:allowableIntervals} Allowable range of path length in [m] for each agent-goal pair. Assigned pairs are indicated with box framing, and pairs violating the stopping condition of Algorithm~\ref{alg:approach} are highlighted in yellow.}
  \centering
  \begin{tabular}{c|| r@{,}l r@{,}l r@{,}l || r@{,}l r@{,}l r@{,}l || r@{,}l r@{,}l r@{,}l ||}
    &\multicolumn{6}{c||}{{\bf Iteration 1 (n = 1024)}} & \multicolumn{6}{c||}{{\bf Iteration 2 (n = 4168)}} & \multicolumn{6}{c||}{{\bf Iteration 3 (n = 16672)}} \\ 
    & \multicolumn{2}{c}{Goal 1} & \multicolumn{2}{c}{Goal 2} & \multicolumn{2}{c||}{Goal 3} & \multicolumn{2}{c}{Goal 1} & \multicolumn{2}{c}{Goal 2} & \multicolumn{2}{c||}{Goal 3} & \multicolumn{2}{c}{Goal 1} & \multicolumn{2}{c}{Goal 2} & \multicolumn{2}{c||}{Goal 3} \\
    \cline{8-9} \cline{14-15} 
    Agent 1  &  ($-\infty$  & $\infty$) &  [\hl{1.94}  & $\infty$) &  [\hl{1.94} & $\infty$) & \multicolumn{1}{|r@{,}}{ ($-\infty$ } & \multicolumn{1}{@{}l|}{ \hl{1.47}] } &  [1.66 & $\infty$) & ($-\infty$ & $\infty$) & \multicolumn{1}{|r@{,}}{ ($-\infty$} & \multicolumn{1}{@{}l|}{ 1.61] } & [1.80 & $\infty$) & ($-\infty$ & $\infty$) \\ 
    \cline{8-9} \cline{14-15} \cline{2-3} \cline{10-11} \cline{16-17}  
    Agent 2  & \multicolumn{1}{|r@{,}}{ ($-\infty$} &  \multicolumn{1}{@{}l|}{\hl{1.22}] } &  [\hl{1.94}  & $\infty$) &  [\hl{1.94} & $\infty$) & ($-\infty$ & $\infty$) & \multicolumn{1}{|r@{,}}{ ($-\infty$ } & \multicolumn{1}{@{}l|}{ \hl{1.66}] } & ($-\infty$ & $\infty$) & ($-\infty$ & $\infty$) & \multicolumn{1}{|r@{,}}{ ($-\infty$ } & \multicolumn{1}{@{}l|}{1.80] } & ($-\infty$ & $\infty$)\\
    \cline{2-3} \cline{10-11} \cline{16-17} \cline{4-5}    
    Agent 3  & ($-\infty$ & $\infty$) & \multicolumn{1}{|r@{,}}{ [\hl{1.22} } & \multicolumn{1}{@{}l|}{ \hl{1.94}] } & [\hl{1.94} & $\infty$) &  [1.47 & $\infty$) & [\hl{1.66} & $\infty$) & [1.47 & $\infty$) & [1.61 & $\infty$) & [1.80 & $\infty$) &  [1.59 & $\infty$) \\
    \cline{4-5} \cline{6-7} \cline{12-13} \cline{18-19}    
    Agent 4  &  ($-\infty$ & $\infty$) & ($-\infty$ & $\infty$) & \multicolumn{1}{|r@{,}}{ ($-\infty$ } & \multicolumn{1}{@{}l|}{1.94] } & [1.47 & $\infty$) & [\hl{1.66} & $\infty$) & \multicolumn{1}{|r@{,}}{ ($-\infty$ } & \multicolumn{1}{@{}l|}{1.46] } & [1.61 & $\infty$) & [1.80 & $\infty$) & \multicolumn{1}{|r@{,}}{ ($-\infty$ } & \multicolumn{1}{@{}l|}{1.59] } \\
    \cline{6-7} \cline{12-13} \cline{18-19}   
    Agent 5  & ($-\infty$ & $\infty$) & [\hl{1.94} & $\infty$) & [\hl{1.94} & $\infty$) & [\hl{1.47} & $\infty$) & [1.66 & $\infty$) & [1.47 & $\infty$) & [1.61 & $\infty$) & [1.80 & $\infty$) & [1.59 & $\infty$) 
  \end{tabular}
\end{table*} 

In the first iteration the roadmaps are coarse, see Figure~\ref{fig:case_study_1}. Given the resulting small value of the scaling factor, $\beta$, and the significant difference between $\free^s$ and $\free^{s - \delta}$, we observe that the upper and lower bounds are not tight enough to satisfy the stopping criteria on Line~9 of Algorithm~\ref{alg:approach}, i.e. $[l_{i,j},u_{i,j}] \not\subseteq [w_{i,j} - \lambdal_{i,j},w_{i,j}+\lambdau_{i,j}]$, for several agent-goals pairs $(i,j) \in \ant$ highlighted with yellow in Figure~\ref{fig:case_study_1} and Table~\ref{tab:allowableIntervals}. In the second iteration, the roadmaps are generated from more samples. As shown in Figure~\ref{fig:case_study_2}, the paths are therefore smoother, the obstacles have been expanded slightly more by shrinking the margin parameter $\delta$, and the scaling factor, $\beta$, has increased. We observe that the candidate assignment has changed in comparison to the first iteration, see the framed cells of Table~\ref{tab:allowableIntervals}. But still, some of the path bounds lie outside of the allowable range, shown again in yellow. The sample size is therefore increased further. In the third iteration, illustrated in Figure~\ref{fig:case_study_3}, the bounds on the optimal paths have converged enough to certify the candidate assignment as optimal and the algorithm terminates. 

\begin{table}
    \caption{\label{tab:values_diffvalues} Evolving variable values in case study.}
  \centering
  \begin{tabular}{c c || r || r || r || r ||}
     & $n$ &  1024 & 4168 & 16672 &  66688\\
     & $\hat{D}$ & 0.0706m & 0.0353m & 0.0177m & 0.0088m\\
     \hline\hline
     \multirow{3}{*}{$\zeta = 0.1$}
     & \cellcolor{black!10} $\delta$ & \cellcolor{black!10} 0.290m &  \cellcolor{black!10} 0.270m & \cellcolor{black!10} 0.252m & \cellcolor{black!10} 0.235m\\
     & \cellcolor{black!20} $r$ & \cellcolor{black!20} 0.211m & \cellcolor{black!20}  0.219m & \cellcolor{black!20} 0.215m & \cellcolor{black!20} 0.206m \\
     \multirow{3}{*}{$\eta = 0.1$} & \cellcolor{black!30} $\beta$ & \cellcolor{black!30} 0.332 & \cellcolor{black!30} 0.677 & \cellcolor{black!30} 0.836 & \cellcolor{black!30} 0.914\\
     & \cellcolor{black!40} $t$ & \cellcolor{black!40} 1.5s & \cellcolor{black!40} 13.0s & \cellcolor{black!40} 155.4s & \cellcolor{black!40} 2075.3s \\
     & \cellcolor{black!50} $Q$ & \cellcolor{black!50} {\tt false} & \cellcolor{black!50} {\tt false} & \cellcolor{black!50} {\tt true} & \cellcolor{black!50} {\tt true}  \\
     \hline\hline
     \multirow{3}{*}{$\zeta = 0.1$}
     & \cellcolor{black!10} $\delta$ & \cellcolor{black!10} 0.290m &  \cellcolor{black!10} 0.270m & \cellcolor{black!10} 0.252m & \cellcolor{black!10} 0.235m\\
     & \cellcolor{black!20} $r$ & \cellcolor{black!20} 0.180m & \cellcolor{black!20}  0.153m & \cellcolor{black!20} 0.135m & \cellcolor{black!20} 0.122m \\
     \multirow{3}{*}{$\eta = 0.5$} & \cellcolor{black!30} $\beta$ & \cellcolor{black!30} 0.216 & \cellcolor{black!30} 0.538 & \cellcolor{black!30} 0.738 & \cellcolor{black!30} 0.855\\
     & \cellcolor{black!40} $t$ & \cellcolor{black!40} 1.2s & \cellcolor{black!40} 6.2s & \cellcolor{black!40} 71.0s & \cellcolor{black!40} 864.6s \\
     & \cellcolor{black!50} $Q$ & \cellcolor{black!50} {\tt false} & \cellcolor{black!50} {\tt false} & \cellcolor{black!50} {\tt false} & \cellcolor{black!50} {\tt true}  \\
     \hline\hline
     \multirow{3}{*}{$\zeta = 0.5$}
     & \cellcolor{black!10} $\delta$ & \cellcolor{black!10} 0.252m &  \cellcolor{black!10} 0.178m & \cellcolor{black!10} 0.126m & \cellcolor{black!10} 0.089m\\
     & \cellcolor{black!20} $r$ & \cellcolor{black!20} 0.178m & \cellcolor{black!20} 0.136m & \cellcolor{black!20} 0.101m & \cellcolor{black!20} 0.074m \\
     \multirow{3}{*}{$\eta = 0.1$} & \cellcolor{black!30} $\beta$ & \cellcolor{black!30} 0.204 & \cellcolor{black!30} 0.480 & \cellcolor{black!30} 0.651 & \cellcolor{black!30} 0.762 \\
     & \cellcolor{black!40} $t$ & \cellcolor{black!40} 1.2s & \cellcolor{black!40} 5.6s & \cellcolor{black!40} 35.6s & \cellcolor{black!40} 285.8s \\
     & \cellcolor{black!50} $Q$ & \cellcolor{black!50} {\tt false} & \cellcolor{black!50} {\tt false} & \cellcolor{black!50} {\tt false} & \cellcolor{black!50} {\tt false}  \\
     \hline\hline
  \end{tabular}
\end{table}

\subsection{Design Choices and Comparisons}\label{sec:comparisions}
\paragraph{Tuning parameters}\label{sec:diffParameters}
From Table~\ref{tab:values_diffvalues} we see how different values of the tuning parameters $\eta,\zeta\in (0,1)$, affect the performances. For larger $\eta$ the connection radius, $r$, follows the lower bound of its range defined in Assumption~\ref{assum:rightRadius} more closely. For larger $\zeta$, the margin $\delta$ decreases faster which means that the obstacles get extended more quickly but also decreases $r$. Lower values of $r$ lead to smaller values of the scaling factor, $\beta$, which in the considered case makes the algorithm require more iterations to certify. However, because a smaller $r$ means checking fewer samples, the computation time for a given sample size decreases. 

\paragraph{Allowable interval bounds}
We compare the method chosen for obtaining the lexicographic largest allowable perturbations from~\cite{Michael2022EJOR} against an alternative method for {\tt AllowableIntervals} that considers a uniform bound on all perturbations \cite{Sotskov1995DAM}. Note that by construction uniform allowable intervals are smaller or equal to the lexicographic largest ones. Using the uniform allowable intervals for the scenario above results in the uncertainty being too large for 5 robot-goal pairs and no certificate being returned. 

\paragraph{Sampling scheme}
Because triangular tiling provides the lowest possible dispersion for two-dimensional configurations spaces \cite{LaValle2006Book}, we know that alternative sampling schemes lead to looser bounds on the shortest path lengths. If we consider randomised sampling as is used in \ac{prm}$^*$, the probabilistic dispersion bounds decrease slower than deterministic ones, see~\cite{Janson2015IJoR}, and for the considered scenario no certificate can be found within the maximal sample size limit, i.e., for $n=66688$, we have
$\delta = 0.268$m, $r=0.219$m, $\beta = 0.70$, $t = 2440.6$s, and $Q = 0$.

\paragraph{Preformance analysis} To evaluate the proposed method we compare it to two naive approaches. We define the \emph{simple-naive} approach to be a bottleneck assignment with the weight equal to the upper bounds on the path lengths found with the accuracy given by $\nmin$ whereas the \emph{complex-naive} approach is defined by making a bottleneck assignment with weights equal to the upper bounds found with the accuracy given by $\nmax$. For the scenario above, we observe that the assignment obtained with the simple-naive approach is not optimal when considering the weights used for the complex-naive approach. Because the assignment obtained with Algorithm~\ref{alg:approach} is certified for $n = \nmax$, we know that the complex-naive assignment is optimal.  

The benefit of starting the planning with low accuracy and increasing it until a certificate is reached compared to just considering the maximal accuracy can be measured in the time saved. Note that the time saved is negative if the assignment is certified in the last iteration or not certified at all. In the considered scenario the time saved normalised by the time it takes for the complex-naive approach is 93\%.

\subsection{Randomised configurations}\label{sec:randomConfiguations}
To analyse the approach in a more general setting we consider randomised examples. Different maps are generated in the configuration space $\cspace = [-1\textrm{m},1\textrm{m}]^2$ with obstacles consisting of balls with randomised radii and centres. Only obstacle configurations that satisfy Assumption~\ref{assum:existsPath} are considered. The position of the robots and goals is also randomised within the set of safe positions. Table~\ref{tab:maps} lists the results of applying Algorithm~\ref{alg:approach} to the randomised configurations with the same tuning parameters as in Section~\ref{sec:results} expect $\nmin = 310$ and $\nmax = 19840$. For each listed choice of the number of agents $|\agents|$, goals $|\tasks|$, and obstacles $m$, and the required safety distance, $s$, 100 simulations are considered. The results include statistics on how often the assignment is certified (Certification), how often the simple-naive assignment is not optimal for the complex-naive paths (Simple fails), and the average normalised time saved by iteratively increasing the sample size (Savings).

\begin{table}
    \caption{\label{tab:maps} Results for randomised configurations.}
  \centering
  \begin{tabular}{r r r r | r r r}
    $|\agents|$ & $|\tasks|$ & m & $s$ &Certification & Simple fails & Savings \\
     \hline
     3 & 2 & 5 & 0.30m & 75\% & 5\% & 59\% \\
     5 & 3 & 3 & 0.30m & 78\% & 5\% & 61\% \\
     5 & 3 & 5 & 0.25m & 67\% & 1\% & 38\% \\
     5 & 3 & 5 & 0.30m & 73\% & 6\% & 55\% \\
     5 & 3 & 5 & 0.35m & 80\% & 4\% & 62\% \\
     5 & 3 & 7 & 0.30m & 82\% & 8\% & 62\% \\
     7 & 4 & 5 & 0.30m & 78\% & 5\% & 61\% 
  \end{tabular}
\end{table} 

\section{Conclusions}\label{sec:conclusions}

We introduced a method that certifies if polynomial-time approximations of the shortest obstacle-avoiding paths between robots and goal are sufficient to guarantee bottleneck optimal goal assignment. If no certificate is returned, the accuracy of the polynomial approximation procedure is increased and repeated until either sufficiently tight bounds are computed or a predefined maximum complexity is reached.  

Numerical examples demonstrated the possibility of certifying the optimality of the goal assignment using only estimates for the optimal path lengths. We observed however that to achieve tight bounds, a large number of samples may be required which quickly leads to computational and storage challenges. In future work, the approach should be implemented efficiently and tested in real-world environments and benchmarking tools such as \cite{Chamzas2022RAL}. Our current work focuses on alternative methods to bound the shortest paths given some knowledge about the obstacles. Extending research should investigate the optimal adaptation of the planning accuracy between iterations and identify conditions under which a certificate of optimality is returned in finite time. Furthermore, we are interested in extending this work to the case of safety rather than shortest path objectives~\cite{Tihanyi2023ECC}.

\bibliographystyle{IEEEtran}
\bibliography{CertifiedBottleneck}

\end{document}